%% file: sample_paper.tex
\documentclass[twoside]{article}

\usepackage[accepted]{aistats2026}

\makeatletter
\renewcommand{\AISTATS@appearing}{%
Preprint.%
}
\renewcommand{\Notice@String}{\AISTATS@appearing}
\makeatother
\usepackage[round]{natbib}

\usepackage{graphicx}
\usepackage{bm}
\usepackage{xcolor}
\usepackage{color}
\usepackage{amsthm}
\usepackage{amssymb}

\definecolor{cbblue}{RGB}{1,115,178}

\usepackage[colorlinks=true,
            linkcolor=cbblue,
            citecolor=cbblue,
            urlcolor=cbblue]{hyperref}

\newtheorem{theorem}{Theorem}[section]

\input{macros.tex}

\begin{document}

%

%

\twocolumn[

\aistatstitle{A Predictive View on Streaming Hidden Markov Models}

\aistatsauthor{ Gerardo Duran-Martin}

\aistatsaddress{ Oxford-Man Institute, University of Oxford } 
]


\begin{abstract}
We develop a predictive-first optimisation framework for streaming hidden Markov models. Unlike classical approaches that prioritise full posterior recovery under a fully specified generative model, we assume access to regime-specific predictive models whose parameters are learned online while maintaining a fixed transition prior over regimes. Our objective is to sequentially identify latent regimes while maintaining accurate step-ahead predictive distributions. Because the number of possible regime paths grows exponentially, exact filtering is infeasible. We therefore formulate streaming inference as a constrained projection problem in predictive-distribution space: under a fixed hypothesis budget, we approximate the full posterior predictive by the forward-KL optimal mixture supported on $S$ paths. The solution is the renormalised top-$S$ posterior-weighted mixture, providing a principled derivation of beam search for HMMs. The resulting algorithm is fully recursive and deterministic, performing beam-style truncation with closed-form predictive updates and requiring neither EM nor sampling. Empirical comparisons against Online EM and Sequential Monte Carlo under matched computational budgets demonstrate competitive prequential performance.
\end{abstract}

\section{Introduction}

Hidden Markov models (HMMs) provide a classical framework for modelling regime changes in sequential data.
Traditional formulations emphasise latent-state recovery or parameter estimation, typically in offline settings \citep{bishop2006pattern}.
In many modern applications, however, data arrive sequentially, requiring HMMs to be deployed in an online or streaming setting \citep{gama2010knowledge}.

In predictive applications, the quantity of interest is the one-step-ahead posterior predictive distribution.
However, exact maintenance of this predictive
in an HMM
requires tracking all possible latent regime paths, whose number grows exponentially in time, rendering exact filtering computationally infeasible.

To address this tension between predictive accuracy and computational tractability,
we adopt a predictive-first perspective:
rather than treating regime identification as the primary objective, we directly target sequential predictive performance under a fixed hypothesis budget.

We formalise this as a constrained optimisation problem over predictive mixtures under a fixed hypothesis budget.
Specifically, we approximate the full posterior predictive by a mixture supported on at most $S$ latent paths and minimise the forward Kullback--Leibler divergence.
We show that the optimal solution retains the $S$ largest posterior path weights and renormalises them, thereby recovering beam search as the solution of an explicit projection problem.

The resulting algorithm is fully recursive and deterministic.
It combines beam search with closed-form regime-specific updates,
yielding analytic predictive mixtures without EM iterations or sampling.

Code can be found at \url{https://github.com/gerdm/streaming-hmm}.

\section{Related work}

Hidden Markov and Markov-switching models are widely used for regime analysis in sequential data. 
In macroeconomics, regime-switching models were popularised by \citet{hamilton1989new}. 
HMMs remain central in speech recognition and signal processing \citep{rabiner2002tutorialhmm}; 
a comprehensive methodological review is provided by \cite{mor2021hmmreview}.

Sequential inference in HMMs typically proceeds via expectation–maximisation (EM) and its online variants \citep{cappe2011onlinehmm}, or via sequential Monte Carlo (SMC) methods such as Rao–Blackwellised particle filters and particle learning \citep{murphy2001rbpf, carvalho2010pl}. 
EM performs recursive likelihood optimisation, whereas particle methods approximate the path posterior through stochastic sampling. 
Both approaches target posterior estimation under a fully specified generative model.

Beam search, or top-$S$ truncation, is a classical complexity-control strategy in sequential decoding for HMMs and related models \citep{ney1987hmmsearch, graves2012sequence}. 
It is commonly introduced as a heuristic approximation to MAP (Viterbi) decoding \citep{viterbi2003error},
retaining only the highest-scoring partial paths at each step. 
In contrast, we maintain a normalised empirical distribution over retained paths rather than a single MAP trajectory.

Our work relates to predictive-first approaches to Bayesian inference, where sequential predictive performance,
rather than posterior recovery,
defines the objective \citep{fong2023martingale,holmes2023statistical,mclatchie2025predictively}. 
This includes both Bayesian models and predictive methods that do not maintain explicit finite-dimensional parameter posteriors, such as Gaussian-process and martingale posterior predictives. 
In particular,
Gaussian processes have previously been combined with switching structures for change-point detection \citep{garnett2009sequential, saatcci2010gpbocd, altamirano2023robustgpbocd} and in variational GP state-space models \citep{frigola2014variational}. 
In contrast, we use regime-specific GPs directly as predictive components within a fully recursive decoding scheme (Section~\ref{sec:experiment-gphmm}).

\section{Streaming hidden Markov model}
\label{sec:streaming-hmm}
We present a streaming HMM (SHMM) that combines beam-style pruning over latent paths with recursive Bayesian updates of regime parameters along each retained path. While these components are individually standard, they are rarely coupled in a single fully recursive procedure aimed at maintaining predictive mixtures. In the following section, we show that this coupling arises naturally as a predictive KL projection, yielding a deterministic streaming algorithm with explicit hypothesis-budget guarantees.

\subsection{Problem formulation}

Let $y_t \in \mathbb{R}$ be the observation at time $t$ and $\cY_t=(y_1,\dots,y_t)$.
We consider a $K$-state HMM with latent regimes $z_t \in \cK=\{1,\dots,K\}$ and path $\cZ_t=(z_1,\dots,z_t) \in \cK^t$.
Regimes evolve according to a known transition matrix $\pi$,
\begin{equation}\label{eq:transition-prior}
    p(z_t \mid \cZ_{t-1}) = p(z_t \mid z_{t-1}) = \pi_{z_{t-1},z_t}.
\end{equation}
Each regime $k$ is equipped with a (possibly nonparametric) predictive model.
For a given path $\cZ_t$, we assume access to regime-specific predictive summaries $\{\vb_{t,k}\}_{k=1}^K$, where $\vb_{t,k}$ encodes the information in the subsequence
\[
\cY_k(\cZ_t)\triangleq (y_\tau : 1\le \tau \le t,\ z_\tau = k).
\]
The corresponding regime-$k$ one-step predictive density is denoted
\[
f_{\vb_{t,k}}(y)
\triangleq
p(y_{t+1}=y \mid \cY_k(\cZ_t)).
\]

Our primary object is the one-step-ahead posterior predictive,
\begin{equation}
\label{eq:posterior-predictive-hmm}
\begin{aligned}
p(y_{t+1} \mid \cY_t)
&=
\sum_{\cZ_t}
p(\cZ_t \mid \cY_t)
\sum_{z_{t+1}}
p(z_{t+1} \mid \cZ_t)\,
f_{\vb_{t,z_{t+1}}}(y_{t+1}).
\end{aligned}
\end{equation}
Here, the posterior predictive is a mixture over paths,
where each path contributes a transition-weighted mixture of regime-specific predictive densities determined by its summaries $\{\vb_{t,k}\}$.
Recursive evaluation of \eqref{eq:posterior-predictive-hmm}
requires maintaining the path posterior $p(\cZ_t \mid \cY_t)$ and, for each path,
the ability to evaluate the corresponding path-conditional predictives
$f_{\vb_{t,k}}(y)$.

\subsection{Sequential belief update}
\label{sec:belief-update}

Let $\cZ_t = \cZ_{t-1} \cup \{z_t\}$ denote an extension of a path with $z_t \in \cK$.  
Recall that each path $\cZ_t$ carries regime-specific predictive summaries
$\{\vb_{t,j}\}_{j=1}^K$, where $\vb_{t,j}$ encodes the information in
$\cY_j(\cZ_t)$ and determines the regime-$j$ predictive density
$f_{\vb_{t,j}}(\cdot)$.

When $z_t = k$, only the $k$-th summary is updated using the new observation $y_t$, while all other summaries remain unchanged:
\begin{equation}
\vb_{t,j} =
\begin{cases}
\texttt{Update}(\vb_{t-1,k}, y_t) & j = k,\\
\vb_{t-1,j} & j \neq k.
\end{cases}
\end{equation}
Here $\texttt{Update}$ denotes a recursive update operator.
In conjugate parametric models, $\vb_{t,k}$ corresponds to the regime-specific posterior, yielding a closed-form recursion (Appendix~\ref{app:parametric-updates}).
More generally, it may implement robust or generalised Bayesian updates under model misspecification \citep{duran-martin2024wolf,altamirano2024robustgp}, or predictive-only updates that do not maintain explicit parameter posteriors \citep{hahn2018recursive,duran2025martingale}.

\subsection{Sequential path posterior update}
\label{sec:path-posterior-update}

At time $t$, the path posterior takes the form
\[
p(\cZ_t \mid \cY_t)
=
p(\cZ_{t-1} \mid \cY_{t-1})\;
p(z_t \mid \cZ_{t-1})\;
f_{\vb_{t-1, z_{t}}}(y_t),
\]
with known initial condition $p(\cZ_0 \mid \cY_0) = p(z_0)$.
Thus, updating the path posterior at time $t$ requires: 
(i) access to the path posterior at time $t-1$ , 
(ii) the transition probability from $z_{t-1}$ to $z_t$ given by \eqref{eq:transition-prior}, and 
(iii) the regime-specific posterior predictive evaluated at $y_t$.

The number of possible paths grows exponentially in $t$;
for $K$ regimes, there are $K^t$ paths at time $t$,
making exact maintenance of $p(\cZ_t \mid \cY_t)$ intractable.

\subsection{Beam search}
\label{sec:beam-search}
A classical heuristic 
to control complexity
is to keep the
$S \ge 1$ paths with most mass at each time step.
Let $\left\{\cZ_{t-1}^{(s)}, w_{t-1}^{(s)}\right\}_{s=1}^S$
denote the retained paths and normalised weights at time $t-1$, defining the empirical posterior
\[
\nu_{t-1}(\cZ_{t-1})
\triangleq
\sum_{s=1}^S
w_{t-1}^{(s)}
\,\delta(\cZ_{t-1}-\cZ_{t-1}^{(s)}).
\]
At time $t$,
each retained path branches to $K$ candidates.
For $k \in \cK$, the unnormalised weight of the extension
$\cZ_t^{(s,k)} = \cZ_{t-1}^{(s)} \cup \{k\}$ is
\[
\hat{w}_t^{(s,k)}
=
w_{t-1}^{(s)}\;
p(k \mid \cZ_{t-1}^{(s)})\;
f_{\vb_{t-1}, k}(y_t).
\]
This yields $S \times K$ candidate paths.

We retain the $S$ candidates with largest $\hat{w}_t^{(s,k)}$ and renormalise:
\[
w_t^{(s)}
=
\frac{\hat{w}_t^{(s)}}{\sum_{j=1}^S \hat{w}_t^{(j)}},
\]
defining the truncated posterior
\[
\nu_t(\cZ_t)
=
\sum_{s=1}^S
w_t^{(s)}\,
\delta(\cZ_t - \cZ_t^{(s)}).
\]
All non-retained paths are assigned zero mass.

This procedure corresponds to beam search (beam pruning) applied to sequential probabilistic models
\citep{graves2012sequence}.

\paragraph{Remark.}
Together, the recursive belief updates (Section~\ref{sec:belief-update})
and the sequential path-posterior recursion (Section~\ref{sec:path-posterior-update})
define a fully recursive streaming HMM.
The resulting truncated posterior yields both one-step and multi-step predictive distributions, as illustrated in Section~\ref{sec:experiments}.
Further details on sequential update strategies are provided in \citet{duran-martin2025bone}.

\section{Beam search as a predictive KL projection}

Here, we show that the beam search used in the streaming hidden Markov model (Section \ref{sec:beam-search})
follows directly from a constrained predictive optimisation problem.
In particular,
we show that
under a budget of $S$ possible paths,
the
forward-KL-optimal one-step-ahead approximation is equivalent to retaining the
$S$ posterior paths with the largest weight
and renormalising them.

\begin{theorem}
\label{theorem:hmm}
Fix $t$. Let
\[
w_t(\cZ_t)\triangleq p(\cZ_t\mid \cY_t)
\]
denote the posterior path weights.

For each path $\cZ_t$, define the \emph{path-conditional predictive}
\[
f_{\cZ_t}(y)
\triangleq
\sum_{k=1}^K
p(z_{t+1}=k \mid \cZ_t)\;
f_{\vb_{t,k}}(y),
\]
where $f_{\vb_{t,k}}(y)=p(y_{t+1}=y\mid \cY_k(\cZ_t))$ is the regime-$k$
predictive density determined by the path-specific summaries $\vb_{t,k}$.

Then the one-step-ahead posterior predictive is
\[
p(y)\equiv p(y_{t+1}=y\mid \cY_t)
=
\sum_{\cZ_t} w_t(\cZ_t)\,f_{\cZ_t}(y).
\]
Let $A \subseteq \cK^t$ be any subset of $S$ paths and define
\[
\delta(A)\triangleq \sum_{\cZ_t\notin A} w_t(\cZ_t),
\qquad
W_A\triangleq \sum_{\cZ_t\in A} w_t(\cZ_t)=1-\delta(A).
\]

Consider predictive approximations supported on $A$,
\[
q(y)=\sum_{\cZ_t\in A}\alpha(\cZ_t)\,f_{\cZ_t}(y),
\]
with $\alpha(\cZ_t)\ge 0$ and $\sum_{\cZ_t\in A}\alpha(\cZ_t)=1$.

Assume that the $\chi^2$-divergence between the discarded and retained predictive mixtures is finite:
\begin{equation}
\label{eq:assumption-chi2}
\tag{A1}
\chi^2(p_{A^c}\,\|\,p_A)
=
\int \frac{(p_{A^c}(y))^2}{p_A(y)}\,dy - 1
\le C < \infty,
\end{equation}
where
\[
\begin{aligned}
p_A(y)
&=
\sum_{\cZ_t\in A}\frac{w_t(\cZ_t)}{W_A}\,f_{\cZ_t}(y),\\
p_{A^c}(y)
&=
\sum_{\cZ_t\notin A}\frac{w_t(\cZ_t)}{\delta(A)}\,f_{\cZ_t}(y).
\end{aligned}
\]

Define the renormalised truncated predictive
\begin{equation}
\label{eq:truncated-predictive}
    q_A(y)
    =
    \sum_{\cZ_t\in A}\frac{w_t(\cZ_t)}{W_A}\,f_{\cZ_t}(y).
\end{equation}

Then:
\begin{enumerate}
\item[(i)] Among all $q$ supported on $A$,
\[
q_A \in \arg\min_q \mathrm{KL}(p\|q).
\]

\item[(ii)] Under \eqref{eq:assumption-chi2},
\begin{equation}\label{eq:bound-predictive}
\mathrm{KL}(p\|q_A)
\le
\log\!\big(1+\delta(A)C\big).
\end{equation}

\item[(iii)] The bound is increasing in $\delta(A)$; hence the optimal support of size $S$ consists of the $S$ largest posterior weights $w_t(\cZ_t)$.
\end{enumerate}
\end{theorem}

See Appendix \ref{sec:proof} for the proof.

\paragraph{Remarks.}
The discarded mass $\delta(A)$ quantifies the posterior weight removed by truncation; when $\delta(A)$ is small, the truncated predictive remains close to the full mixture, as controlled by Bound~\eqref{eq:bound-predictive}.
Thus,
theorem \ref{theorem:hmm} identifies beam search as the forward-KL optimal projection onto $S$-supported predictive mixtures.


\section{Experiments}
\label{sec:experiments}

\subsection{GP-HMM}
\label{sec:experiment-gphmm}

We consider a two-regime HMM with regime-specific Gaussian Process (GP) emissions. 
Data are generated from a process with regime-dependent linear drift and oscillatory structure:
\[
y_t = y_{t-1}
+ \text{slope}^{(k)} \Delta t
+ w_1^{(k)} \sin\!\big(w_2^{(k)} t_{\text{local}}\big)\Delta t
+ \varepsilon_t,
\]
where $k\in\{0,1\}$ and $\varepsilon_t\sim\mathcal{N}(0,\sigma^2)$. 
Slopes are $\pm 0.15$ and both regimes share $(w_1,w_2)=(0.3,0.5)$.

Each regime uses a GP predictive with a Gaussian plus periodic kernel \citep{williams2006gaussian,duvenaud2014gp}. 

Figure~\ref{fig:gp-dgp} shows one-step-ahead predictions. 
Uncertainty increases near regime switches and contracts within stable segments.

\begin{figure}[htb]
\centering
\includegraphics[width=0.6\columnwidth]{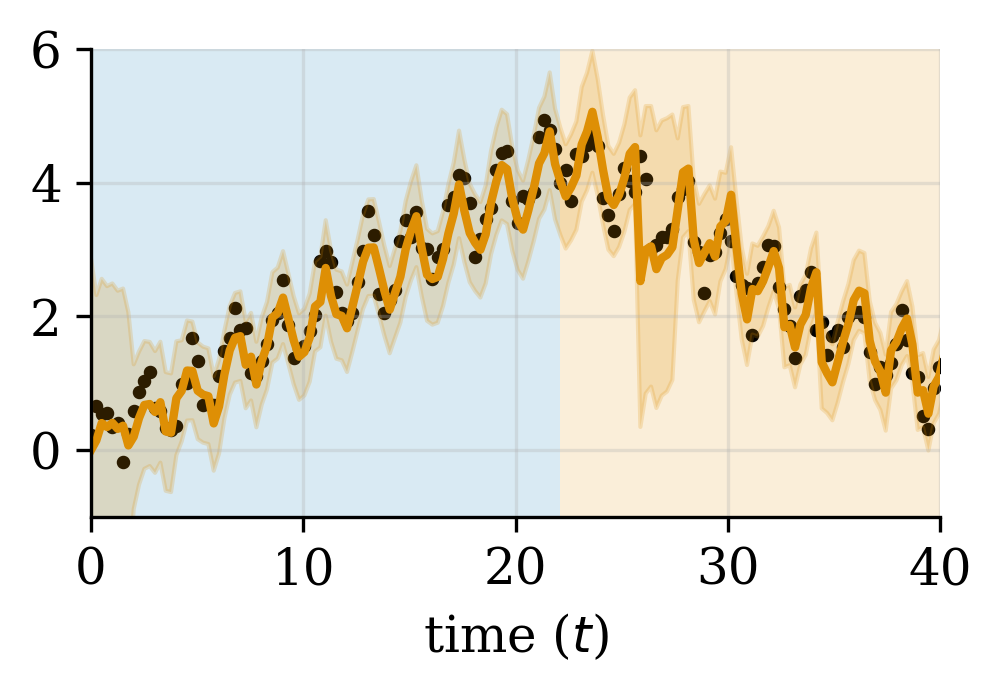}
\vspace{-1em}
\caption{GP-HMM one-step-ahead predictions.}
\label{fig:gp-dgp}
\end{figure}

Figure~\ref{fig:gp-multi-step-forecast} shows multi-step forecasts that are
close to and far from a regime change.

\begin{figure}[htb]
\centering
\includegraphics[width=\columnwidth]{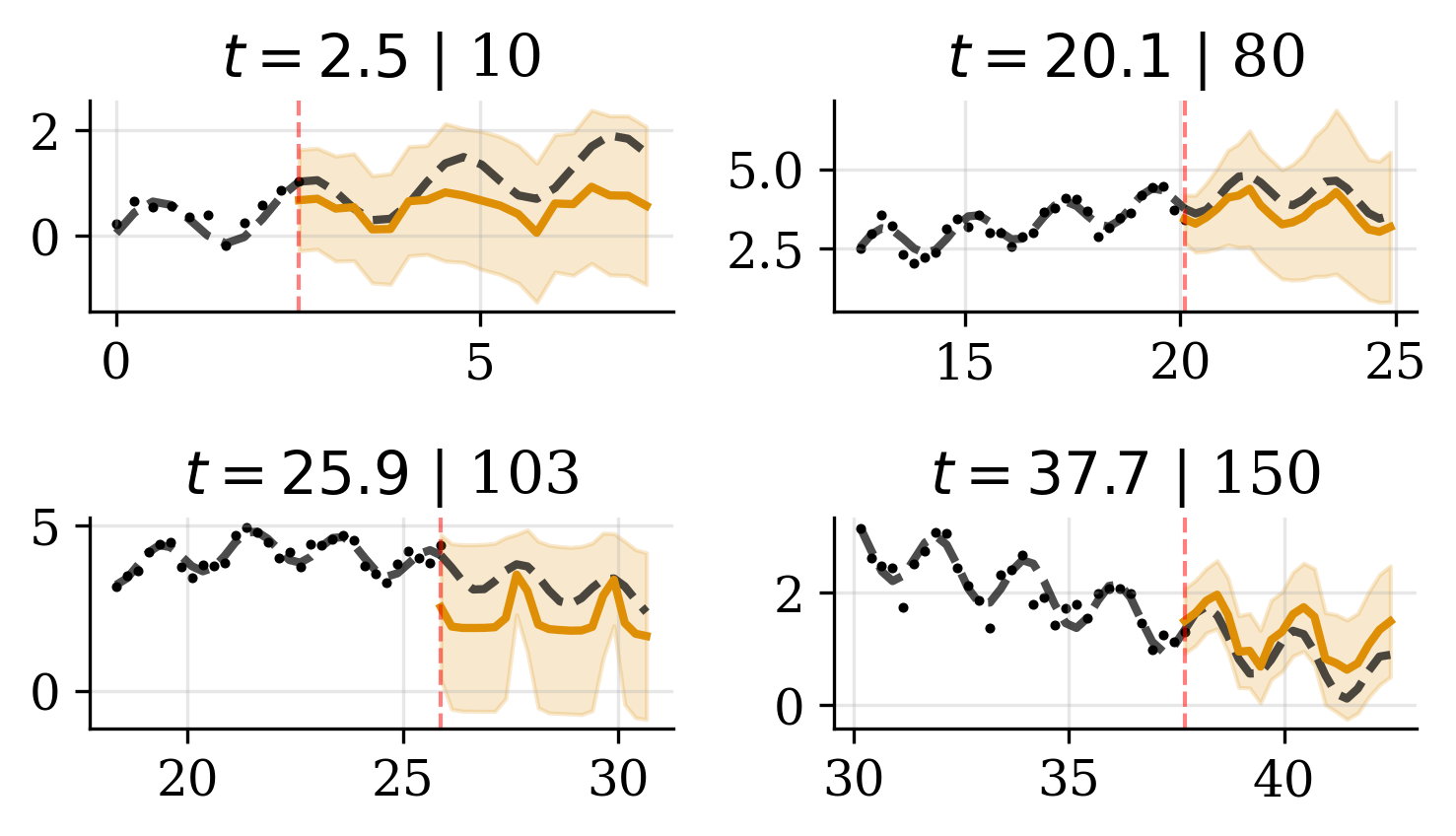}
\vspace{-2em}
\caption{Representative multi-step forecasts.}
\label{fig:gp-multi-step-forecast}
\end{figure}

\subsection{1D Gaussian HMM}

We next consider a three-regime Gaussian HMM with known variance and unknown regime means. 
We compare SHMM against online EM \citep{cappe2011onlinehmm} and a Rao–Blackwellised particle filter (RBPF) \citep{murphy2001rbpf}
with prior as proposal.
For matched computational budgets, SHMM and RBPF use $S=2$ hypotheses.

Figure~\ref{fig:mean-level-estimate} shows the regime-mean estimates. With $S=2$, RBPF fails to consistently track regime switches. Online EM recovers mean levels but exhibits higher variance. SHMM remains stable and tracks regime changes accurately under the same hypothesis budget.

\begin{figure}[htb]
\centering
\includegraphics[width=0.8\columnwidth]{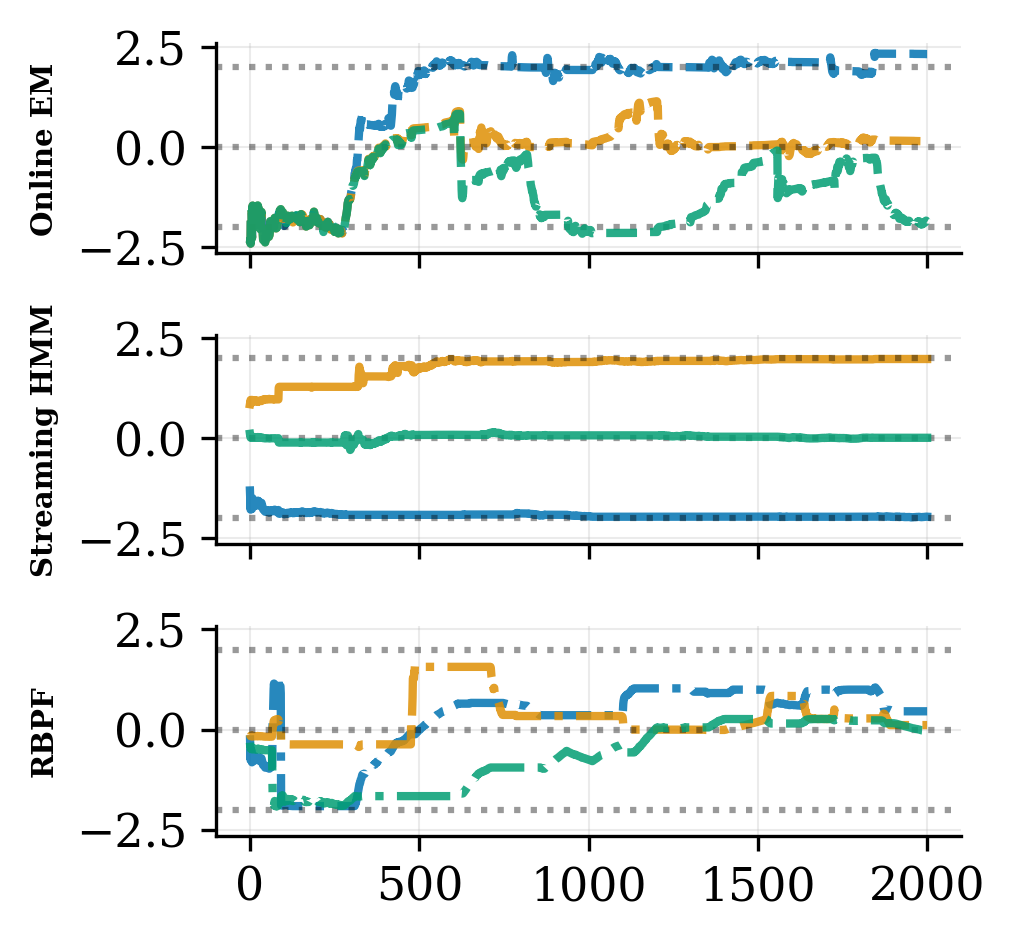}
\vspace{-1em}
\caption{Estimated regime means.}
\label{fig:mean-level-estimate}
\end{figure}

Table~\ref{tab:results} reports predictive accuracy and runtime (mean $\pm$ std). 
SHMM achieves the lowest MAE and RMSE while remaining computationally competitive.
Further results with varying $S$ are shown in Appendix \ref{sec:1d-gauss-further}.

\begin{table}[htb]
\centering
\scriptsize
\begin{tabular}{lccc}
\textbf{Method} & \textbf{MAE} & \textbf{RMSE} & \textbf{Time (s)} \\
\hline
Online EM & $0.9 \pm 0.01$ & $1.2 \pm 0.02$ & $\mathbf{0.04 \pm 0.01}$ \\
SHMM      & $\mathbf{0.8 \pm 0.02}$ & $\mathbf{1.1 \pm 0.02}$ & $0.10 \pm 0.03$ \\
RBPF      & $1.0 \pm 0.04$ & $1.3 \pm 0.06$ & $0.26 \pm 0.01$ \\
\end{tabular}
\caption{Predictive performance and runtime.}
\label{tab:results}
\end{table}

\bibliographystyle{plainnat}   
\bibliography{refs}

\appendix
\onecolumn

\aistatstitle{A Predictive View on Streaming Hidden Markov Models: \\
Supplementary Materials}

\section{Proofs}
\label{sec:proof}

\begin{proof}
Fix a set $A$. Define
\[
p_A(y)\triangleq \sum_{\cZ_t\in A}\frac{w_t(\cZ_t)}{W_A}\,f_{\cZ_t}(y),
\qquad
p_{A^c}(y)\triangleq \sum_{\cZ_t\notin A}\frac{w_t(\cZ_t)}{\delta(A)}\,f_{\cZ_t}(y),
\]
so that the posterior predicive is
\begin{equation}
p(y)=W_A\,p_A(y)+\delta(A)\,p_{A^c}(y).
\label{eq:mixture-decomposition}
\end{equation}
By definition, $p_A=q_A$ in \eqref{eq:truncated-predictive}.

\medskip
\noindent\textbf{(i) Optimal weights for fixed support.}

Let $q$ be any predictive supported on $A$, i.e.
$q(y)=\sum_{\cZ_t\in A}\alpha(\cZ_t)f_{\cZ_t}(y)$ with
$\alpha(\cZ_t)\ge 0$ and $\sum_A \alpha(\cZ_t)=1$.
Then
\[
\mathrm{KL}(p\|q)
=
\int p(y)\log\frac{p(y)}{q(y)}\,dy.
\]
Insert $p_A$:
\[
\log\frac{p(y)}{q(y)}
=
\log\frac{p(y)}{p_A(y)}
+
\log\frac{p_A(y)}{q(y)},
\]
hence
\begin{equation}
\mathrm{KL}(p\|q)
=
\int p(y)\log\frac{p(y)}{p_A(y)}\,dy
+
\int p(y)\log\frac{p_A(y)}{q(y)}\,dy.
\label{eq:kl-split}
\end{equation}
The first term does not depend on $q$.
Using the decomposition \eqref{eq:mixture-decomposition} in the second term,
\begin{equation}
\int p(y)\log\frac{p_A(y)}{q(y)}\,dy
=
W_A\int p_A(y)\log\frac{p_A(y)}{q(y)}\,dy
+
\delta(A)\int p_{A^c}(y)\log\frac{p_A(y)}{q(y)}\,dy.
\label{eq:weighted-decomposition}
\end{equation}
The first term equals $W_A\,\mathrm{KL}(p_A\|q)\ge 0$, with equality if and only if $q=p_A$.
Therefore $q=p_A=q_A$ minimises $\mathrm{KL}(p\|q)$ among all $q$ supported on $A$.

\medskip
\noindent\textbf{(ii) Predictive bound.}

Since $q_A=p_A$, using \eqref{eq:mixture-decomposition},
\[
\frac{p(y)}{p_A(y)}
=
W_A+\delta(A)\frac{p_{A^c}(y)}{p_A(y)},
\]
so
\begin{equation}
\mathrm{KL}(p\|p_A)
=
\int p(y)\log\!\left(
W_A+\delta(A)\frac{p_{A^c}(y)}{p_A(y)}
\right)\,dy.
\label{eq:kl-pa}
\end{equation}

Assume that the $\chi^2$-divergence between $p_{A^c}$ and $p_A$ is finite, i.e.
\[
{\chi^2(p_{A^c}\,\|\,p_A)
=
\int \frac{(p_{A^c}(y))^2}{p_A(y)}\,dy - 1
\le C < \infty.}
\]

{By Jensen's inequality applied to \eqref{eq:kl-pa},}
\[
\mathrm{KL}(p\|p_A)
\le
\log\!\left(
\int p(y)\Big[
W_A+\delta(A)\frac{p_{A^c}(y)}{p_A(y)}
\Big]dy
\right).
\]

{Since $\int p(y)\,dy=1$, this becomes}
\[
\mathrm{KL}(p\|p_A)
\le
\log\!\left(
W_A
+
\delta(A)
\int p(y)\frac{p_{A^c}(y)}{p_A(y)}dy
\right).
\]

{Using the mixture decomposition \eqref{eq:mixture-decomposition},}
\[
\int p(y)\frac{p_{A^c}(y)}{p_A(y)}dy
=
W_A
\int p_A(y)\frac{p_{A^c}(y)}{p_A(y)}dy
+
\delta(A)
\int p_{A^c}(y)\frac{p_{A^c}(y)}{p_A(y)}dy.
\]

{The first term equals $W_A$.
The second term can be bounded using the $\chi^2$ condition:}
\[
\int p_{A^c}(y)\frac{p_{A^c}(y)}{p_A(y)}dy
=
\int \frac{(p_{A^c}(y))^2}{p_A(y)}dy
\le 1 + C.
\]

{Therefore}
\[
\int p(y)\frac{p_{A^c}(y)}{p_A(y)}dy
\le
W_A + \delta(A)(1+C).
\]

{Substituting into the Jensen bound yields}
\[
\mathrm{KL}(p\|p_A)
\le
\log\!\left(
W_A
+
\delta(A)\big[W_A + \delta(A)(1+C)\big]
\right).
\]

{Using $W_A=1-\delta(A)$ and simplifying gives}
\[
\mathrm{KL}(p\|q_A)
\le
\log\!\big(1 + \delta(A)C\big).
\]

\medskip
\noindent\textbf{(iii) Choice of support.}

The upper bound depends on $A$ only through $\delta(A)$.
Since $C\ge 0$, the function
$\log(1+\delta(A) C)$ is increasing in $\delta\in[0,1)$.
Therefore minimising the bound is equivalent to minimising $\delta(A)$,
i.e.\ maximising
\[
W_A=\sum_{\cZ_t\in A} w_t(\cZ_t).
\]
This is achieved by selecting the $S$ largest posterior weights $w_t(\cZ_t)$.
\end{proof}

\section{Parametric updates}
\label{app:parametric-updates}

As a convenient special case, consider a regime-parameterised emission family with parameters
$\theta_k$, collected as $\Theta_K=(\theta_1,\dots,\theta_K)$.
Conditional on $(\cZ_t,\Theta_K)$, observations are independent with regime-specific emissions
\[
p(y_t \mid z_t,\Theta_K)
=
p(y_t \mid \theta_{z_t}),
\qquad
p(\cY_t \mid \cZ_t,\Theta_K)
=
\prod_{\tau=1}^t p(y_\tau \mid \theta_{z_\tau}).
\]
We assume independent priors across regimes,
\[
p(\Theta_K)=\prod_{k=1}^K p(\theta_k).
\]

\paragraph{Posterior factorisation.}
Conditioned on a path $\cZ_t$, the parameter posterior factorises across regimes:
\begin{equation}
\label{eq:param-post-factorised}
p(\Theta_K \mid \cZ_t, \cY_t)
=
\prod_{k=1}^K p(\theta_k \mid \cY_k(\cZ_t)),
\end{equation}
where
\[
\cY_k(\cZ_t)
=
(y_\tau : 1\le \tau \le t,\ z_\tau = k)
\]
denotes the observations assigned to regime $k$ under $\cZ_t$.
Equivalently,
\[
p(\Theta_K \mid \cZ_t, \cY_t)
=
\prod_{k=1}^K
\left(
p(\theta_k)
\prod_{\tau=1}^t
p(y_\tau \mid \theta_k)^{1\{z_\tau=k\}}
\right).
\]
Thus each retained path maintains regime-specific posteriors that are updated only when the corresponding regime is active.

\paragraph{Sequential update.}
If $\cZ_t=\cZ_{t-1}\cup\{k\}$, only the $k$-th posterior changes:
\[
p(\theta_k \mid \cY_k(\cZ_t))
\propto
p(\theta_k \mid \cY_k(\cZ_{t-1}))\,p(y_t \mid \theta_k),
\qquad
p(\theta_j \mid \cY_j(\cZ_t))
=
p(\theta_j \mid \cY_j(\cZ_{t-1}))
\ \text{for } j\neq k.
\]
Under conjugate priors this recursion is closed form and reduces to a constant-time sufficient-statistics update.

\paragraph{Path-conditional predictive.}
Given the regime-$k$ posterior under $\cZ_t$, the one-step-ahead predictive is
\[
p(y_{t+1} \mid z_{t+1}=k,\cY_t,\cZ_t)
=
\int
p(y_{t+1}\mid \theta_k)\,
p(\theta_k \mid \cY_k(\cZ_t))
\,\mathrm{d}\theta_k,
\]
which is available in closed form for exponential-family models with conjugate priors and for linear-Gaussian emissions.

\section{Further experimental results}
\subsection{1D Gaussian HMM}
\label{sec:1d-gauss-further}
Figure~\ref{fig:prediction-comparison} compares one-step-ahead predictions across methods for $S=2$. 
SHMM closely tracks regime changes and stabilises rapidly after transitions. 
RBPF exhibits delayed adaptation and occasionally interpolates between regimes when particle diversity collapses. 
Online EM produces smoother trajectories but with increased variance near switching points.

\begin{figure}[htb]
\centering
\includegraphics[width=0.6\columnwidth]{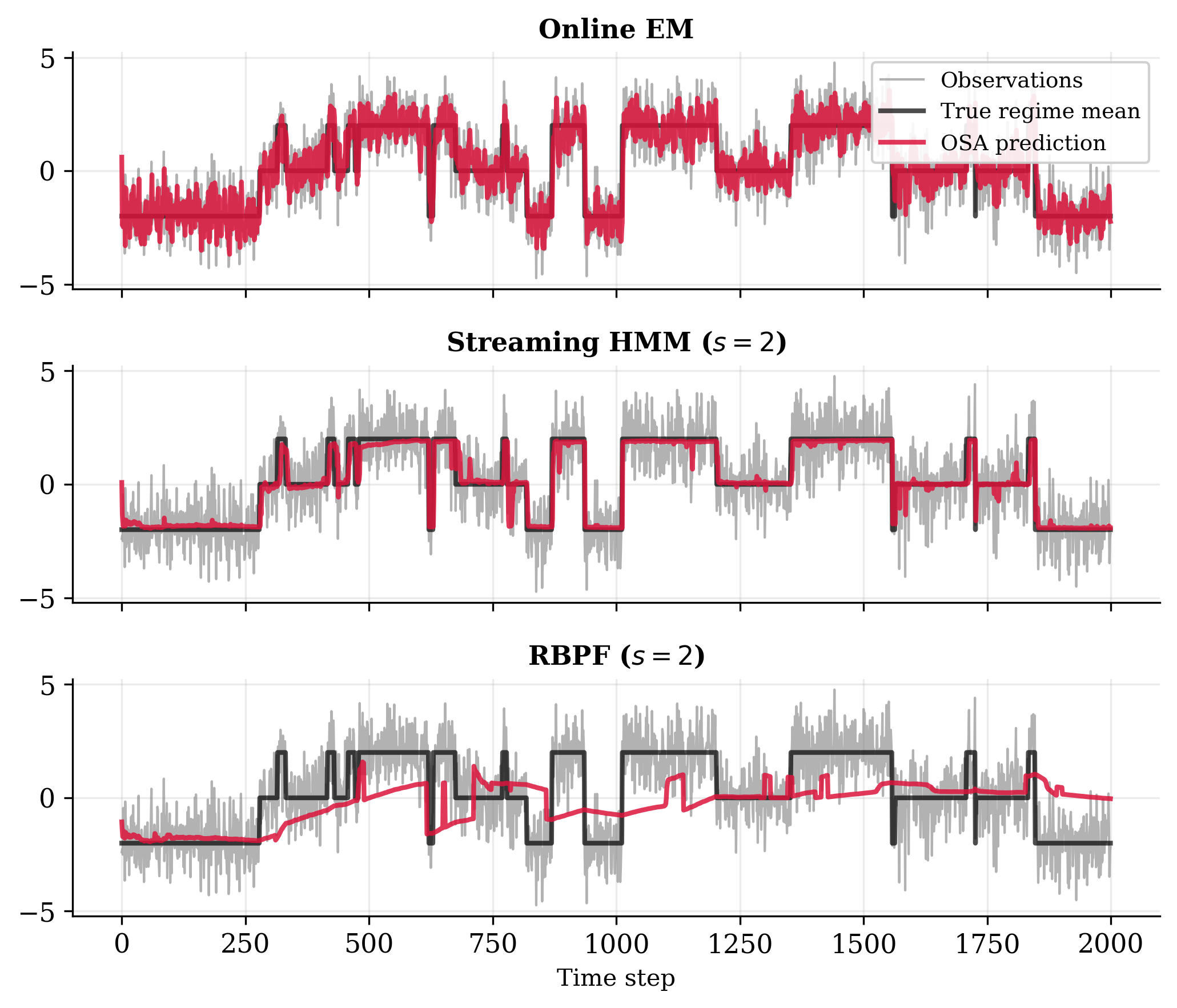}
\caption{One-step-ahead predictions for the two-regime Gaussian HMM with $S=2$. Grey denotes observations, black the true regime mean, and red the predictive mean. SHMM adapts rapidly to regime changes, while RBPF and Online EM exhibit delayed or smoother transitions.}
\label{fig:prediction-comparison}
\end{figure}

Figure~\ref{fig:varying-particles} examines predictive accuracy and runtime as the particle count $S$ varies. 
SHMM attains low MAE and RMSE with small $S$ and exhibits only marginal gains beyond $S=5$, indicating that most posterior mass is captured by a small hypothesis set. 
In contrast, RBPF requires larger $S$ to approach comparable accuracy, reflecting the variance induced by stochastic resampling. 
Runtime increases approximately linearly with $S$ for both methods, though SHMM remains consistently faster across the evaluated range.
\begin{figure}[htb]
\centering
\includegraphics[width=\columnwidth]{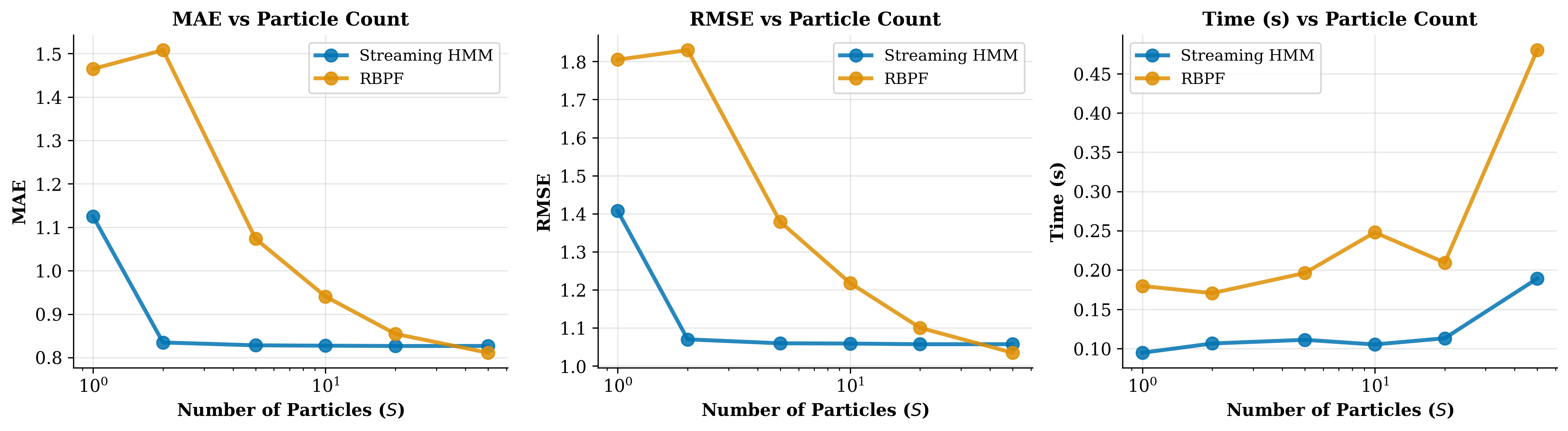}
\caption{Predictive accuracy (MAE, RMSE) and runtime as a function of particle count $S$ for SHMM and RBPF. SHMM attains low error with small $S$, while RBPF requires larger $S$ for comparable accuracy. Runtime increases approximately linearly with $S$ for both methods.}
\label{fig:varying-particles}
\end{figure}

\vfill

\end{document}

%% file: macros.tex
\newcommand{\vb}{\myvec{b}}

\def\vb{{\bm{b}}}

\newcommand{\cK}{\mathcal{K}}
\newcommand{\cY}{\mathcal{Y}}
\newcommand{\cZ}{\mathcal{Z}}